\documentclass[twoside,leqno,twocolumn,10pt]{article}  

\usepackage{ltexpprt} 

\usepackage{amssymb}
\usepackage{latexsym}
\usepackage{amsmath}
\usepackage[dvips]{color}
\usepackage{epsfig}
\usepackage{url}

\usepackage{float}
\usepackage{xspace}
\usepackage{algorithmwh}

\newtheorem{definition}{Definition}
\newtheorem{observation}{Observation}

\def\avg{\ensuremath{\mathit{avg}}}

\begin{document}

\title{Specious rules: an efficient and effective unifying method for removing misleading and uninformative patterns in association rule mining}

\author{Wilhelmiina H\"am\"al\"ainen$^1$\\
\and 
Geoffrey I.\ Webb$^2$\\
\small $^1$Department of Computer Science, Aalto University, Finland – wilhelmiina.hamalainen@aalto.fi\\
\small $^2$Faculty of IT, Monash University, Melbourne, Australia –
geoff.webb@monash.edu}
\date{}

\maketitle


\begin{abstract} 
We present theoretical analysis and a suite of tests and procedures for addressing a broad class of redundant and misleading association rules we call \emph{specious rules}. Specious dependencies, also known as \emph{spurious}, \emph{apparent}, or \emph{illusory associations}, refer to a well-known phenomenon where marginal dependencies are merely products of interactions with other variables and disappear when conditioned on those variables.
The most extreme example is Yule-Simpson's paradox where two variables present positive dependence in the marginal contingency table but
negative in all partial tables defined by different levels of a confounding factor. It is accepted wisdom that in data of any nontrivial dimensionality it is infeasible to control for all of the exponentially many possible confounds of this nature.  In this paper, we consider the problem of specious dependencies in the context of statistical association rule mining.  We define specious rules and show they offer a unifying framework which covers many types of previously proposed redundant or misleading association rules. After theoretical analysis, we introduce practical algorithms for detecting and pruning out specious association rules efficiently under many key goodness measures, including mutual information and exact hypergeometric probabilities. We demonstrate that the procedure greatly reduces the number of associations discovered, providing an elegant and effective solution to the problem of association mining discovering large numbers of misleading and redundant rules.

\end{abstract}


 \section{Introduction}

Association rule mining is often a useful first step for any data mining
task, because it reveals the statistical dependence structure of
data. A major issue in association rule mining is that analyses often
produce excessive numbers of rules. Even if rules are carefully
pruned with strict statistical measures, the results are likely to
contain large numbers of redundant, unproductive, uninteresting or
even misleading and paradoxical rules. 

In this paper, we investigate the problem of {\em specious
  associations}, also known as spurious, apparent, illusory, or
misleading associations (e.g., \cite{yule1903}). Specious associations
refer to a well-known phenomenon where marginal dependencies appear as
products of interactions with other variables and disappear when
conditioned on those variables. The most extreme example is
Yule-Simpson's paradox where two variables present positive dependence
in the marginal contingency table but negative in all partial tables
defined by different levels of a confounding factor. As far as we
know, the problem has not previously been addressed in
the context of association rule discovery, but there are approaches (e.g.,
\cite{kingfkais,HuangWebb2005,causalrules,webbml}) that prune out
certain subtypes of specious rules, together with other search
constraints.

We begin from basic definitions, tests and properties of specious rules.  
We show that specious rules offer a unifying framework that includes
many previously proposed forms of redundant or misleading association rules as
their special cases. Thus, their removal can drastically reduce the
number of discovered associations. On the other hand, efficient
detection of specious rules offers a solution to the oft-cited problem
of misleading associations in traditional data analysis. We give new
theoretical properties that help to identify specious dependencies
efficiently, without testing all possible (exponentially many)
confounding factors. Based on these, we introduce an efficient
algorithm for detection and removal of specious dependencies.
We report experiments on mining non-specious, statistical
association rules with mutual information, and draw final
conclusions.

\section{Definitions and concepts}

\subsection{Unconditional and conditional dependencies}

According to a classical definition, two events $ X{=}x$ and $C{=}c$
are statistically independent, if $P( X{=}x,C{=}c)=P(
X{=}x)P(C{=}c)$. Otherwise, they are considered dependent. A commonly
used measure for the strength of this unconditional (or marginal) dependence is
leverage $\delta( X{=}x,C{=}c)=P( X{=}x,C{=}c)-P(
X{=}x)P(C{=}c)$. When $X$ and $C$ are (composed or single) binary
variables, $\delta(X,C)=\delta(\neg X,\neg C)=-\delta(X,\neg
C)=-\delta(\neg X,C)$, 
where $X$ and $C$ are used as shorthands for $X{=}1$ and $C{=}1$ and $\neg X$ and $\neg C$ as shorthands for $X{=}0$ and $C{=}0$ respectively. 
Therefore, the dependency can be represented by
giving just one value pair, usually one expressing positive
dependency.

In pattern discovery, statistical dependencies can be expressed as a
type of association rules, which we call {\em statistical association
rules} or {\em dependency rules} (see e.g., \cite{kingfkais}). Such
rules express always statistical dependencies but, unlike traditional association rules, do not have any
minimum frequency requirements. Instead,
there are requirements for the strength and/or significance of
statistical dependencies. We concentrate on dependency
rules of the form $X\rightarrow C{=}c$, $c\in
\{0,1\}$, where $X$ is a set of positive valued binary attributes and
$C$ is a single binary attribute. Rule $X\rightarrow C{=}1$ (or simply
$X\rightarrow C$) expresses a positive dependency between $X$ and $C$
while $X\rightarrow C{=}0$ (or $X\rightarrow \neg C$) expresses a
negative dependency. We note that $X\rightarrow C$, $C\rightarrow X$,
$\neg X \rightarrow \neg C$, and $\neg C \rightarrow \neg X$ express
the same dependency.

When two dependency rules are compared, it is often useful to study the conditional dependence of one rule (e.g., $Q\rightarrow C$) on another with the same consequent attribute (e.g., $X\rightarrow C$). This is equivalent to studying conditional dependencies between $Q$ and $C$ given $X$ or $\neg X$. Now $Q$ and $C$ are conditionally independent given $ X{=}x$, $x\in \{0,1\}$, if 
\begin{equation*}P( X{=}x,Q,C)=\frac{P( X{=}x,Q)P( X{=}x,C)}{P( X{=}x)}.
\end{equation*}

Otherwise they are considered conditionally dependent. To measure the strength of the 
conditional dependence, we extend the leverage measure. 
\textit{Conditional leverage} of $Q\rightarrow C$ given value $X$ is 
$$\delta_c(Q,C|X)=P(X,Q,C)-P(X,Q)P(C|X)$$
and given value $\neg X$ it is 
$$\delta_c(Q,C|\neg X)=P(\neg X,Q,C)-P(\neg X,Q)P(C|\neg X).$$
The sign of the conditional leverage determines whether the given conditional dependency is positive or negative.

\subsection{Specious dependency rules}

In this paper, we are interested in detecting and pruning out specious
dependency rules.  These are rules that occur only as side-products of
other dependencies. Ultimately, specious dependency rules are a
phenomenon of the unknown population from which the sample data is
drawn. Therefore, we give the definition in the population level and
explain how statistical significance testing can be used to detect
specious rules from sample data.

\begin{definition}[Specious dependency rule]
Let $X$ and $Q$ be sets of binary attributes and $C$ a single binary attribute. Dependency rule $Q\rightarrow C{=}c$ ($c\in\{0,1\}$) expressing positive dependency between $Q$ and $C{=}c$ is specious if there
is another rule $X\rightarrow C{=}c_x$ ($c_x\in \{0,1\}$) such that 
$\delta_c(Q,C{=}c|X)\leq 0$ and $\delta_c(Q,C{=}c|\neg X)\leq 0$ in the population.
\end{definition}

Figure \ref{spectypes} shows four alternative ways that rule 
$X\rightarrow C{=}c$ or $X\rightarrow C{\neq}c$ 
can make $Q\rightarrow C{=}c$, $c\in \{0,1\}$ 
specious. Note that in a and d, the specious rule is  $Q\rightarrow C$ and in b and c, it is $Q\rightarrow \neg C$. 

\begin{figure*}[!t]
\begin{center}
\includegraphics[width=\textwidth]{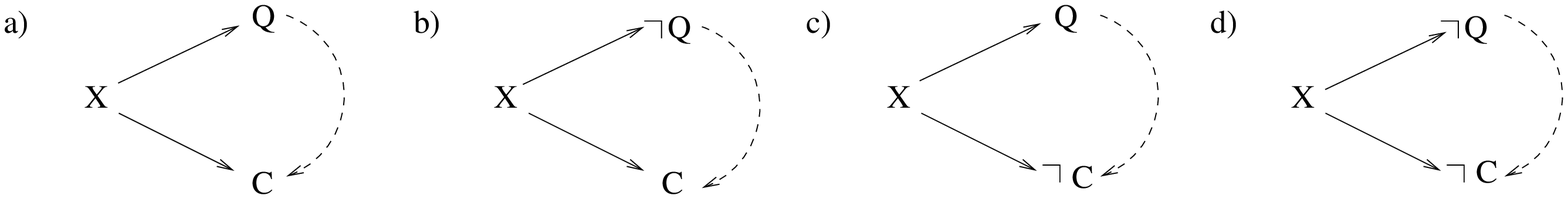}
\caption{Four ways that $X\rightarrow C{=}c$ or $X\rightarrow C{\neq}c$ can make $Q\rightarrow C{=}c$, $c\in \{0,1\}$ specious. The solid lines show genuine statistical dependencies, while the dotted lines show derived dependencies that arise as their side-products.}
\label{spectypes}
\end{center}
\end{figure*}

In the sample level, speciousness of $Q\rightarrow C{=}c$ can be detected by
studying the conditional leverages $\delta_1=\delta_c(Q,C|X)$ and
$\delta_2=\delta_c(Q,C|\neg X)$ and statistical significance of the
conditional dependency. If both $\delta_1\leq 0$ and $\delta_2\leq 0$,
then the observed dependency $Q\rightarrow C{=}c$ is completely due to
variable $X$ and either disappears or changes its direction when
conditioned. However, sample data is only an incomplete presentation
of the underlying population, and it is possible that both conditional
leverages are weakly positive, even if the rule is specious. For detecting such specious rules, one has to test the
significance of the observed deviations from conditional
independence. If the deviations are not significant with a desired
level, then the rule can be considered specious. 

\subsection{Statistical and information-theoretic tests}

Birch's exact test \cite{birch} and mutual information are two robust methods for 
 evaluating significance of partial dependencies and detecting specious rules.
Birch's exact test defines the exact hypergeometric
probability that dependency $Q\rightarrow C{=}c$ is at least as strong as observed, if
$Q$ and $C$ were actually conditionally independent 
given $X$ (null hypothesis). In the original article, Birch gives only the
hypergeometric point probability $P(n_{qc}~|~n,
n_x,n_{xc},n_{xq},n_{q},n_c)$ that the frequency of $QC{=}c$ is exactly
$n_{qc}$, when data size ($n$) and frequencies of $X$, ($X,C{=}c$), ($X,Q$), $Q$
and $C{=}c$ ($n_x$, $n_{xc}$, $n_{xq}$, $n_{q}$, and $n_c$) are fixed. The
corresponding cumulative probability ($p$-value) is obtained by
summing up all point probabilities where the dependency between $Q$
and $C{=}c$ is at least as strong as observed (which is equivalent to
$N_{qc}\geq n_{qc}$ under the null hypothesis of conditional
independence). Now, the test for the speciousness of rule
$Q\rightarrow C{=}c$, given 
$X\rightarrow C{=}c_x$ ($c_x=c$ or $c_x=1-c$) becomes

\begin{multline}
\label{bircheq}
p_B=P(N_{qc}\geq n_{qc}~|~n,n_x,n_{xc},n_{xq},n_{q},n_c)\\
=\frac{\sum_{i+j\geq n_{qc}}\left(n_{xq} \atop i
	\right)\left(n_{x\neg q}) \atop n_{xc}-i\right)
	\left(n_{\neg xq} \atop j
	\right)\left(n_{\neg x\neg q)} \atop n_{\neg xc}-j\right)}{\left(n_x \atop n_{xc} \right)\left(n_{\neg x} \atop
	n_{\neg xc}\right)}.
\end{multline}

One problem is to select a threshold value $\alpha\in ]0,1]$ such that
    rules with $p_B>\alpha$ can be safely pruned out as specious
    (i.e., the null hypothesis of conditional independence is kept). A
    classical threshold in significance testing is $\alpha=0.05$ but a
    smaller or larger value can be used to prune out more or less
    rules as specious.

A corresponding information-theoretic test for speciousness can be constructed using conditional mutual information. With a short-hand notation $p_x=P(X)$ it is defined as 

\begin{multline}
\label{condMI}
MI(Q\rightarrow C|X\rightarrow C)\\
=MI(Q\rightarrow C|X)+MI(Q\rightarrow C|\neg X)\\
=\log \frac{p_x^{p_x}p_{xqc}^{p_{xqc}}p_{xq\neg c}^{p_{xq\neg c}}p_{x\neg qc}^{p_{x\neg qc}}p_{x\neg q\neg c}^{p_{x\neg q\neg c}}}{p_{xq}^{p_{xq}}p_{x\neg q}^{p_{x\neg q}}p_{xc}^{p_{xc}}p_{x\neg c}^{p_{x\neg c}}}\\
+\log \frac{p_{\neg x}^{p_{\neg x}}p_{\neg xqc}^{p_{\neg xqc}}p_{\neg xq\neg c}^{p_{\neg xq\neg c}}p_{\neg x\neg qc}^{p_{\neg x\neg qc}}p_{\neg x\neg q\neg c}^{p_{\neg x\neg q\neg c}}}{p_{\neg xq}^{p_{\neg xq}}p_{\neg x\neg q}^{p_{\neg x\neg q}}p_{\neg xc}^{p_{\neg xc}}p_{\neg x\neg c}^{p_{\neg x\neg c}}}.
\end{multline}

This measure is computationally easy to evaluate and it has also other
attractive properties which help in the search. However, there is one
drawback since $MI$ is always non-negative and, therefore, negative and positive 
conditional dependence cannot be separated. As a
solution, we have modified the measure such that the signs of terms are
reversed in the case of negative conditional dependency (i.e.,
$MI(Q,C|X=x)$ becomes $-MI(Q,C|X=x)$ if $\delta_c(Q,C|X=x)<0$). 
The modified (signed) measure is denoted $MI_S$. 
Once again, one should decide a minimum $MI_S$ value which is required
for non-speciousness.

Similarly, one can design a $\chi^2$-based test for speciousness (like
Mantel-Haenszel's test and its variations that asymptotically
approximate $p_B$), but it has the same problem as $MI$ in separating
positive and negative conditional dependencies.

\section{Relationship to redundancy, equivalence and Yule-Simpson's paradox}
\label{specialcases}

\subsection{Speciousness of specialized and generalized rules}
\label{specase1}

Speciousness is closely related to the concepts of productivity
\cite{webbml} and redundancy \cite{kingfkais} and the question of
whether a more general rule $X\rightarrow C{=}c$ or its specialization
$Q\rightarrow C{=}c$ is superfluous. There are two possible cases,
where $X\rightarrow C{=}c$ makes $Q\rightarrow C{=}c$ specious and
either $X\subset Q$ or $Q\subset X$ (where $\subset$ denotates a
proper subset). For simplicity, we represent the results only for
positive consequent $C{=}1$.

In the first case, rule $X\rightarrow C$ is a generalization of rule
$Q\rightarrow C$, i.e., $X\subset Q$. In the next section we will show
that this phenomenon can occur only if dependency $X\rightarrow C$ is more
significant than $Q\rightarrow C$, with many types of significance measures.
However, this means that $Q\rightarrow C$ is 
superfluous or redundant with respect to $X\rightarrow C$. Now
$\delta_2=0$ and the speciousness depends on $\delta_1$, whether
$\delta_1\leq 0$ or $\delta_1>0$ but the deviation from zero is
insignificant.  If we mark $Q=XZ$ and the corresponding frequency by $n_{xz}$, the exact test for the speciousness of $XZ\rightarrow C$ given $X\rightarrow C$ reduces to
\begin{equation*}
p_B=\frac{\sum_{i\geq n_{xzc}}\left(n_{xz} \atop i
	\right)\left(n_{x\neg z} \atop n_{xc}-i\right)}{\left(n_x \atop n_{xc} \right)}.
\end{equation*}
An interesting coincidence is that this is the same as the test for
the significance of productivity in \cite{webbml}. With mutual information, the test is simply $MI_S=MI(XZ\rightarrow C|X)$ (assuming $\delta_1>0$).

In the second case, rule $X\rightarrow C$ is a specialization of rule
$Q\rightarrow C$, i.e., $Q\subset X$. Once again, it will turn out that $X\rightarrow C$
has to be more significant than $Q\rightarrow C$ to make it specious, given an appropriate significance measure (see Section
\ref{detecting}).
This means that one type of specious rules are
generalizations of non-redundant rules (rules $X\rightarrow C$ which
are better than any $Y\rightarrow C$, $X\subset Y$ with the given 
significance measure). 
Now $\delta_1=0$, and the speciousness depends solely on $\delta_2$. 
If we mark $X=QZ$ and the corresponding frequency by $n_{qz}$, the exact test for the speciousness reduces to 
\begin{equation*}
p_B=\frac{\sum_{j\geq n_{q\neg zc}}\left(n_{q\neg z} \atop j
  \right)\left(n_{\neg q} \atop n_{\neg qc}-j\right)}{\left(n-n_{qz} \atop n_c-n_{qzc}\right)}.
\end{equation*}

With mutual information, the equivalent test is $MI_S(Q\rightarrow
C|QZ\rightarrow C)=MI(Q\rightarrow C|\neg(QZ))=MI(\neg Q\rightarrow
\neg C|\neg(QZ)) $ (assuming $\delta_2>0$). This is the same as the
significance of productivity of $\neg Q\rightarrow \neg C$ with
respect to $\neg (QZ)\rightarrow \neg C$.

A further connection to redundancy testing occurs if we compare a rule
$X\rightarrow C$ and its specialization $XZ\rightarrow C$ using
speciousness tests. If $MI_S(X\rightarrow C|XZ\rightarrow C)=MI(\neg
X\rightarrow \neg C|\neg (XZ))$ is too small, then $X\rightarrow C$ is
specious, and if $MI_S(XZ\rightarrow C|X\rightarrow
C)=MI(XZ\rightarrow C|X)$ is too small, then $XZ\rightarrow C$ is
specious. Since both cannot occur simultaneously, only the rule with a
smaller conditional $MI$ value can be specious. It turns out that this
is the same as comparing the unconditional $MI$ values of rules. Thus,
the test for redundancy in \cite{kingfkais} (i.e., testing if
$MI(XZ\rightarrow C)<MI(X\rightarrow C)$) is a special case of
speciousness testing. It is also equivalent to comparing the
significance of productivity of $XZ\rightarrow C$ with respect to
$X\rightarrow C$ and of $\neg X\rightarrow \neg C$ with respect to
$\neg(XZ)\rightarrow \neg C$.

\subsection{Equivalent rules}

A special case of speciousness occurs when two sets of attributes, $Q$
and $X$, cover the same set of elements, i.e., when
$cov(X)=cov(Q)$. In this case, $X$ and $Q$ are called {\em equivalent}
and they have exactly the same dependency rules. In this special case,
speciousness of $Q\rightarrow C{=}c$ by $X\rightarrow C{=}c$ is of type a
or c in Fig \ref{spectypes}. On the other hand, if $X$ and $\neg Q$
are equivalent, then $X$ and $Q$ have complement rules of the form
$X\rightarrow C{=}c$ and $Q\rightarrow C{\neq}c$. This corresponds to
cases b and d in Fig \ref{spectypes}. 

For a pair of equivalent rules, all conditional leverages
(of $Q\rightarrow C{=}c$ given $X\rightarrow C{=}c$ or $X\rightarrow C{\neq}c$  
and vice versa) are zero. This matches the definition of
speciousness, even if one cannot consider them specious per se.
Still, equivalent rules capture the same information and there is no need to
report both, once the equivalence rule $X\rightarrow Q$ is
reported. This type of speciousness has potentially important
computational implications, because it enables radical pruning of the
search space.

Equivalent rules are also related to the concepts of {\em closed item sets} and
{\em minimal generators} \cite{pasquier99}. An attribute set $X$ is called closed, if for all 
$Y\supset X$ holds $P(Y)<P(X)$, and it is called a minimal generator if for all $Y\subset X$ 
holds $P(Y)>P(X)$. Now if $X$ and $Q$ are closed sets with the
same minimal generator $Z=X\cap Q\neq \emptyset$, then
$cov(X)=cov(Q)=cov(Z)$ and both $X\rightarrow C{=}c$ and $Q\rightarrow C{=}c$ 
are equivalent with $Z\rightarrow C{=}c$ and thus
specious. However, specious rules of this type are also superfluous
specializations explained in the previous subsection and are pruned
as such. We recall that $Z=X\cap Q$ is not necessarily a
generator at all (i.e., $P(X|Z)<1$ and $P(X|Q)<1$), even if $X$ and $Q$ 
were equivalent. In summary, 
all rules $Q\rightarrow C{=}c$ and $X\rightarrow C{=}c$, whose antecedents
have the same minimal generator $Z$, are mutually equivalent and
specious, but all mutually equivalent rules do not have the same
minimal generator.

\subsection{Yule-Simpson's paradox}

Speciousness is closely related to Yule-Simpson's
paradox (see e.g., \cite{Alin2010,simpson1951}). The paradox
refers to the situation where the unconditional dependency between $Q$ and
$C{=}c$ is positive, but becomes negative when conditioned on either $X$
or $\neg X$. In a wider definition, disappearing dependences are also included here, i.e., there is marginal independence but conditional dependence or vice versa, like in the original definition by Yule \cite{yule1903}. In this sense, the core question of speciousness is whether Yule-Simpson's paradox occurs in the population. 
Assuming positive dependency $Q\rightarrow C{=}c$, the paradox can be expressed formally by three conditions

{\footnotesize
\begin{itemize}
\item[(i)] $P(C{=}c|Q)>P(C{=}c|\neg Q)$  i.e. $\delta(Q,C{=}c)>0$;
\item[(ii)] $P(C{=}c|Q,X)\leq P(C{=}c|\neg Q,X)$ i.e. $\delta_c(Q,C{=}c|X)\leq 0$;
\item[(iii)] $P(C{=}c|Q,\neg X)\leq P(C{=}c|\neg Q,\neg X)$ i.e. $\delta_c(Q,C{=}c|\neg X)\leq 0$.
\end{itemize}}
In this case, $Q\rightarrow C{=}c$ is specious by definition. However,
it may easily go unnoticed in the traditional data analysis involving
only pair-wise dependencies. Furthermore, it is possible that the
confounding variable $X$ is not present in the data. 

A classical example demonstrates a specious dependency between a new treatment and recovery (see e.g., \cite{lindley}). In the example, a positive marginal dependency between treatment and recovery is observed when the patients are considered as a whole. However, when the data is stratified by the sex of the patients, it turns out that the new treatment is negatively associated with recovery both for women and men. The explanation is that men were more often selected to have the new treatment and they also had a higher recovery rate. 

One can easily show (by setting $\delta_1\leq 0$ and $\delta_2\leq 0$)
that Yule-Simpson's paradox can occur only if
$$\delta(Q,C)\leq \frac{\delta(X,Q)\delta(X,C)}{P(X)P(\neg X)}.$$
This means that it can be avoided if either $Q$ or $C$ is independent
from $X$.  In the example of treatment and recovery, the paradox would
have been avoided, if the same proportion of men and women had been
allocated to the new treatment, thus removing the dependency between
the sex and treatment. However, this is seldom possible in
retrospective studies or exploratory data analysis and examples of
Yule-Simpson's paradox occur in real world data. One solution to the
problem could be dependency rule analysis with speciousness
detection. If dependency rules were searched from the
treatment--recovery data, it would immediately be clear that the
strongest and most significant dependencies are {\em Male
  $\rightarrow$ Treatment} and {\em Male $\rightarrow$ Recovery},
while {\em Treatment $\rightarrow$ Recovery} is clearly weaker and
relatively insignificant. Still, an exhaustive search would probably
output it, together with many other specious rules, because they do
not check interrelations of dependencies.

\section{Detecting specious associations}
\label{detecting}

\subsection{Theoretical results}

Detecting specious dependency rules reduces to the following core
problem: Given two rules, $r_1$: $Q\rightarrow C{=}c$ and $r_2$:
$X\rightarrow C{=}c_x$, $c,c_x\in \{0,1\}$, is $r_1$ specious by $r_2$,
$r_2$ by $r_1$, or neither? In essence, this is the same as asking:
How probable it is to observe at least as strong $Q\rightarrow C{=}c$, if $Q$ and
$C$ were conditionally independent given dependency $X\rightarrow
C{=}c_x$? If it is more probable than the opposite, observing as strong
$X\rightarrow C{=}c_x$, given $Q\rightarrow C{=}c$, then at most $r_1$ can be
specious. In the opposite case, at most $r_2$ can be specious. In both
cases, it is also required that the probability is sufficiently
high. There is also a special case, where the probabilities are
equal. Then one cannot say that either rule is specious, because we cannot have a circular argument, where $r_1$ is made specious by $r_2$ and $r_2$ by $r_1$. 
(Still, it suffices to report only one of the rules, if they are equivalent
and the equivalence information is also reported.) Therefore, we can
from now on assume that there are only three alternatives: either one
of the rules is specious by the other or none of them is.

A brute-force approach for detecting specious rules would evaluate all
possible rule pairs $r_1$ and $r_2$. However, this is quite an
intractable solution because the universe of all possible rules
defined by $k$ attributes is $\mathcal{O}(2^k)$. As a solution, we will show
that when the rules are searched with appropriate goodness measures, 
then the speciousness of the best $K$ rules (or all
sufficiently good rules) can be checked exhaustively among those $K$
rules, ignoring the rest. The foundation of this powerful result is
the following observation on the relationship between marginal and
conditional independence relations.

\begin{observation}
Let us consider relationships of three binary variables $X$, $Q$, and $C$.
Independence assumptions are presented by the following null hypotheses:
$H_{01}$: $X$ and $C$ independent, $H_{02}$: $Q$ and $C$ independent,
$H_{03}$: $Q$ and $C$ conditionally independent given $X$, 
$H_{04}$: $Q$ and $C$ conditionally independent given $\neg X$,
$H_{05}$: $X$ and $C$ conditionally independent given $Q$,
$H_{06}$: $X$ and $C$ conditionally independent given $\neg Q$.
Then the composed null hypotheses $(H_{01}\wedge H_{03}\wedge H_{04})$ and
$(H_{02}\wedge H_{05}\wedge H_{06})$ are equivalent.
\end{observation}

\begin{proof}
The proof is by showing that both composed null hypotheses describe the same 
data distribution. We recall that the distribution can be represented by triplet $(n_{xc},n_{qc},n_{xqc})$, when $n$, $n_c$, $n_x$, $n_q$, and $n_{xq}$ are given. 

Assuming $H_{01}\wedge H_{03}\wedge H_{04}$ means that
$p_{xc}=p_xp_c$, $p_{xqc}=\frac{p_{xq}p_{xc}}{p_x}=p_{xq}p_c$, and
$p_{\neg xqc}=\frac{p_{\neg xq}p_{\neg xc}}{p_{\neg x}}=p_{\neg xq}p_c$. 
Therefore, $p_{qc}=p_{xqc}+p_{\neg xqc}=(p_{xq}+p_{\neg xq})p_c=p_qp_c$, 
and the triplet is $(p_xp_c,p_qp_c,p_{xq}p_c)$.

Assuming $H_{02}\wedge H_{05}\wedge H_{06}$ means that $p_{qc}=p_qp_c$, 
$p_{xqc}=\frac{p_{xq}p_{qc}}{p_q}=p_{xq}p_c$, and
$p_{x\neg qc}=\frac{p_{x\neg q}p_{\neg qc}}{p_{\neg q}}=p_{x\neg q}p_c$. 
Therefore, $p_{xc}=p_{xqc}+p_{x\neg qc}=(p_{xq}+p_{x\neg q})p_c=p_xp_c$, 
which leads to the same triplet $(p_xp_c,p_qp_c,p_{xq}p_c)$.
\end{proof}

\begin{corollary}
Let $n$, $n_c$, $n_x$, $n_q$, and $n_{xq}$ be given. Then for the probability distributions of $n_{xc}$ and $n_{qc}$ given null hypotheses holds 
\begin{multline*}
P(n_{xc}|H_{01})P(n_{qc}|n_{xc},H_{03}\wedge H_{04})=\\
P(n_{xc},n_{qc}|H_{01}\wedge H_{03}\wedge H_{04})=
P(n_{xc},n_{qc}|H_{02}\wedge H_{05}\wedge H_{06})\\
=P(n_{qc}|H_{02})P(n_{xc}|n_{qc},H_{05}\wedge H_{06})
\end{multline*}
and thus 
\begin{equation}
\label{probratio}
\frac{P(n_{xc}~|~H_{01})}{P(n_{qc}~|~H_{02})}=
\frac{P(n_{xc}~|~n_{qc},H_{05}\wedge H_{06})}{P(n_{qc}~|~n_{xc},H_{03}\wedge H_{04})}.
\end{equation}
\end{corollary}

Equation \ref{probratio} means that for judging whether
$P(n_{xc}~|~n_{qc},H_{05}\wedge H_{06})<P(n_{qc}~|~n_{xc},H_{03}\wedge H_{04})$ 
(i.e., if $r_2$ could make $r_1$ specious), it suffices to check whether
$P(n_{xc}~|~H_{01})<P(n_{qc}~|~H_{02})$. This can be done by fixing
the sampling model and estimating the point probabilities either
directly or indirectly. However, it is not necessary to estimate the
exact probabilities, because most goodness measures for statistical
dependencies reflect the same distributions under the same null
hypotheses. We will shortly show that for speciousness detection, it
suffices that the goodness measure and its conditional counterpart are
{\em order-homomorphic}, as defined in the following:

\begin{definition}
Let $M$ be a goodness measure for dependency rules and $M_c$ its conditional 
variant, such that  
$M(Q\rightarrow C{=}c)$ measures marginal dependence between $Q$ and $C{=}c$ 
and $M_c(Q\rightarrow C{=}c|X\rightarrow C{=}c_x)$ measures conditional dependence
between $Q$ and $C{=}c$ given $X\rightarrow C{=}c_x$. $M$ and $M_c$ are called 
{\em order-isomorphic}, if for all rules $r_1$, $r_2$ 
$$M_c(r_1|r_2)\leq M_c(r_2|r_1) \Leftrightarrow M(r_1)\leq M(r_2)$$
and {\em order-homomorphic} if for all rules $r_1$, $r_2$  
$$M_c(r_1|r_2)\leq M_c(r_2|r_1) \Rightarrow M(r_1)\leq M(r_2).$$
Order-isomorphic measures $M$ and $M_c$ are called {\em distance-preserving}, 
if $M_c(r_1|r_2)-M_c(r_2|r_1)=M(r_1)-M(r_2)$ and {\em ratio-preserving} if 
$\frac{M_c(r_1|r_2)}{M_c(r_2|r_1)}=\frac{M(r_1)}{M(r_2)}.$
\end{definition}

It turns out that many commonly used measures are order-homomorphic or
even order-isomorphic. For example, exact hypergeometric point
probabilities are ratio-preserving ($P$s in Eq. \ref{probratio} are
replaced by hypergeometric probabilities) and mutual information and
log-likelihood ratio (2$MI$) are distance-preserving. For the
$\chi^2$-measure, order-homomorphism may not hold exactly but
asymptotically \cite{lewis1962}. It is an open problem whether
Fisher's $p$-value (cumulative hypergeometric probability) is
order-homomorphic with Birch's $p$-value, like the corresponding point
probabilities. If this could be proved, they would likely be the most
robust measures for finding genuine rules.

Efficient speciousness detection is based on the following simple yet powerful theorem:

\begin{theorem}
Let $(M, M_c)$ be a pair of order-homomorphic measures, where $M$ is a marginal and $M_c$ a conditional goodness measure for dependency rules. We assume that $M$ and $M_c$ are increasing (decreasing) by goodness, i.e., high (low) values indicate significant dependence. In addition, we assume that for any rule $r_1$ which is specious by another rule $r_2$ holds $M_c(r_1|r_2)<M_c(r_2|r_1)$ and 
$M_c(r_1|r_2)\leq \theta$ (for decreasing measures $M_c(r_1|r_2)>M_c(r_2|r_1)$ and $M_c(r_1|r_2)\geq \theta$), where $\theta$ is a predefined threshold. Then dependency rule $r_1$ can be made specious only by such $r_2$ for which $M(r_1)<M(r_2)$ ($M(r_1)>M(r_2)$).
\end{theorem}

\begin{proof}
The proof follows directly from the definitions of speciousness and 
order-homomorphism.
\end{proof}

This has an important consequence from an algorithmic point of view. For simplicity, we give the result only for increasing goodness measures.

\begin{corollary}
Let $R$ be a set of binary attributes and $\mathcal{U}$ a universe of all possible rules $X\rightarrow C{=}c$, $C\in R$, $X\subseteq R\setminus\{C\}$, $c\in \{0,1\}$. Let $(M, M_c)$ be an order-homomorphic pair of increasing goodness 
measures and $\tau$ a minimum threshold for $M$ (alternatively, $M$-value of the $K$th best rule). Further, let    
$\mathcal{R_{\tau}}=\{r~|~r\in \mathcal{U}, M(r)\geq \tau\}$ denote the set of all rules having sufficiently good $M$-value in the given data. Then,  
for detecting speciousness of any $r_1\in \mathcal{R}_{\tau}$, it suffices to 
check $r_2\in \mathcal{R}_{\tau}$, $r_2\neq r_1$, having $M(r_2)\geq M(r_1)$. Furthermore, $r_1$ (or its reverse rule, with the antecedent and consequent reversed) and $r_2$ (or its reverse rule) should have the same attribute as the consequent.
\end{corollary}

In practice, this means that speciousness detection can be done in the
post-processing phase, when only the best $K$ rules or all
sufficiently good rules are available. The only requirement is that
the results set $\mathcal{R}$ contains the best rules with the given
measure or at least all non-specious rules. This is not 
guaranteed, if one has used any suboptimal pruning heuristics like
minimum frequency thresholds, restrictions on the rule complexity, or
exclusion of negative dependency rules ($X\rightarrow \neg C$). In the
latter case, some specious rules may not be detected, even if the confounding factors were represented in the data.


\subsection{Algorithm}

The simplest approach for speciousness detection is to search for
the top-$K$ (positive and negative) dependency rules with an
order-homomorphic measure and then evaluate speciousness of each rule
$r_i$ with respect to better rules $r_j$ in the post-processing
phase. However, this approach typically results in a large number of
specious rules, which are redundant specializations of more generic
rules (Section \ref{specase1}). Therefore, it is desirable to
prune out redundant rules during the actual search phase, which also
accelerates the search remarkably (see e.g., \cite{kingfkais}). 

When the set of best non-redundant dependency rules has been
discovered, the rest of the specious rules can be detected in the
post-processing phase using Algorithm \ref{specdetect}. In the
algorithm, all three special cases of speciousness from Section
\ref{specialcases} as well as the normal case are handled
separately. One reason is that some special cases can be checked with less
computation. For example, there is no need to evaluate frequency
$n_{xq}$ if the rules are equivalent or $X\subset Q$, assuming that
$n_x$, $n_{xc}$, $n_q$, and $n_{qc}$ are available.  Birch's exact
test is also computationally demanding and is preferably performed only
when needed. 

The first case covers equivalent rules which cannot be considered
as specious per se. Still, they contain redundant information which
only complicates understanding. If $X$ and $Q$ are equivalent, then
each $X$'s rule has an equivalent $Q's$ rule; if $X$ and $\neg Q$ are
equivalent, then for each $X\rightarrow C{=}c$, there is $Q\rightarrow
C{\neq}c$. Therefore, it is sufficient to report the equivalence
information ($X\rightarrow Q$ or $X\rightarrow \neg Q$), unless it 
occurs as a dependency rule in the list (which happens when $|X|=1$ or $|Q|=1$).

The second case covers rules of the form $Q\subset X$, where
$X\rightarrow C{=}c$ is a non-redundant specialization of $Q\rightarrow C{=}c$. 
The opposite case where $X\subset Q$ has already been handled 
during the search. Since $\delta_1=0$, one could simply check whether
$\delta_2\leq 0$. However, evaluating significance of speciousness
detects also cases where $\delta_2$ is weakly positive.

The third case corresponds to Yule-Simpson's paradox, where both
$\delta_1$ and $\delta_2$ are non-positive, and it is not necessary to
evaluate the significance of speciousness at all. For example, if the
measure is $MI_S$, then it is known that $MI_{S}\leq 0$. This case
includes also a pathological special case which has to be checked
separately. Namely, when $P(X)=P(C{=}c)=P(X,C{=}c)$, then all rules
$Q\rightarrow C{=}c$ and $Q\rightarrow C{\neq}c$ have
$\delta_1=\delta_2=0$ and would be considered specious. This could potentially 
lose interesting dependencies and 
therefore, it is required that there
remains rule $X\rightarrow Q$ or $X\rightarrow \neg Q$ containing the
same information as $Q\rightarrow C{=}c$.

The fourth case is the normal case where the significance measure is
always evaluated.

\begin{algorithm}[!htb]
\caption{{SpecDetect}($\mathcal{R},K,\theta$) for detecting specious rules among the top-$K$ non-redundant dependency rules in an ordered list $\mathcal{R}$. Here, measure $M_c$ is increasing by goodness and $\theta$ is its minimum threshold.}
\label{specdetect}
\begin{code}
for $i=K$ \uto $2$ \urem{in ascending order}\\
\>take rule $r_i\in \mathcal{R}: Q \rightarrow C{=}c$\\
\>for $j=1$ \uto $i-1$ \urem{in descending order}\\
\>\>take rule $r_j\in \mathcal{R}: X\rightarrow A=a$\\
\>\>\uif ($Q=X$) \uor($Q=A$) \uor($C=X$) \uor($C=A$)\\
\>\>\>reverse rules if needed\\
\>\>\>resulting $Q'\rightarrow C'=c$ and $X'\rightarrow C'=a$\\
\>\>\>\uif ($c\neq a$) \uand ($\delta(X',Q')>0$)\\
\>\>\>\>take next $r_j$\\
\>\>\>\>$\delta_1=\delta_c(Q,C{=}c|X')$; $\delta_2=\delta_c(Q,C{=}c|\neg X')$\\
\>\>\>\ucom{Check four alternatives:}\\
\>\>\>\uif (($Q'$ and $X'$) or ($Q'$ and $\neg X'$) equivalent)\\
\>\>\>\>report equivalence and prune out $r_i$\\
\>\>\>\uif (($Q'\subset X'$) \uand ($c=a$) \uand ($M_c(r_i|r_j)\leq\theta)$)\\
\>\>\>\>prune out $r_i$\\
\>\>\>\uif (($\delta_1\leq 0$) \uand ($\delta_2\leq 0$))\\
\>\>\>\>\uif ($n_x=n_{xc}=n_c$) check pathological case\\
\>\>\>\>\uelse report Y-S paradox and prune our $r_i$\\
\>\>\>\uif ($M_c(r_i|r_j)<\theta$) \urem{normal case}\\
\>\>\>\> report and prune out $r_i$
\end{code}
\end{algorithm}

\section{Experiments}

\subsection{Test setting}

The main goal of the experiments was to explore the nature and extent
of specious rules and how their pruning affects results. For this
purpose we tested classical benchmark data sets with varying
dimensions and densities from the UCI Machine Learning Repository
\cite{MLrep} and the Frequent Itemset Mining Dataset Repository
\cite{FIMI}. The test data is described in Table \ref{testsets}.  All
sets are real world data except T10I4D100K and T40I10D100K which are
synthetic data sets simulating market basket data.

\begin{table}
\centering
\caption{Description of data sets: Abbr=abbreviation, $n$=number of rows, $k$=number of attributes, $tlen$=average transaction length.} 
\label{testsets}
\begin{tabular}{|l|l|r|r|r|}
\hline
Data&Abbr&$n$&$k$&$tlen$\\
\hline
Plants&Plants&22632&70&12.5\\
Chess&Chess&3196&75&37.0\\
Mushroom&Mush&8124&119&23.0\\
Accidents&Acc&340183&468&33.8\\
T10I4D100K&T1&100000&870&10.1\\
T40I10D100K&T4&100000&942&39.6\\
PumsbStar&PStar&49046&1934&50.5\\
Pumsb&Pumsb&49046&2113&74.0\\
\hline
\end{tabular}
\end{table}

In the experiments, we searched for the top-100 and top-1000 positive
or negative dependency rules with mutual information and analyzed the
detected specious rules. The search was done with the Kingfisher
program \cite{kingfkais}, which was extended with a special module for
speciousness testing. The basic version of Kingfisher already prunes
out redundant rules (a special case of specious rules) during the search
and it was used as a baseline for comparison. No minimum frequency
thresholds or other constraints were used, except with Accidents,
which was computationally the most demanding data set. When the top-1000
rules were searched from Accidents, the complexity of rules was
restricted to six attributes (max five attributes in the
antecedent). For speciousness testing with $MI_S$, we used a cautious
threshold $\theta=0.5$. Larger thresholds were tried if no type 3
specious rules were detected. Birch's $p$-value was calculated for
checking that $\theta$ was not too large (i.e., none of the pruned
partial dependencies were significant).

For each specious rule $Q\rightarrow C{=}c$, we recorded its subtype
according to the first $X\rightarrow C{=}c$ or $X\rightarrow C{\neq}c$ 
that made it specious. The subtypes were the same as in Algorithm 
\ref{specdetect}.
In addition, we calculated several statistics for comparing specious
(still non-redundant) and non-specious rules (their frequency, strength
of dependency, rule complexity) and evaluating the degree of
speciousness (conditional leverages, $MI_S$, Birch's $p$, strength of the
mediating rule $X\rightarrow Q{=}q$, $q\in \{0,1\}$).  Type 0 rules were
excluded from statistics, because they are a special case, which
cannot be considered specious per se.

\subsection{Results}

Proportions of specious rules and their distribution to subtypes are given
in Figure \ref{specdistr}. Mean statistics characterizing
non-equivalent specious rules in contrast to non-specious rules are
given in Table \ref{specrules100}. 

\begin{figure}[!ht]
\centering
\includegraphics[width=0.5\textwidth]{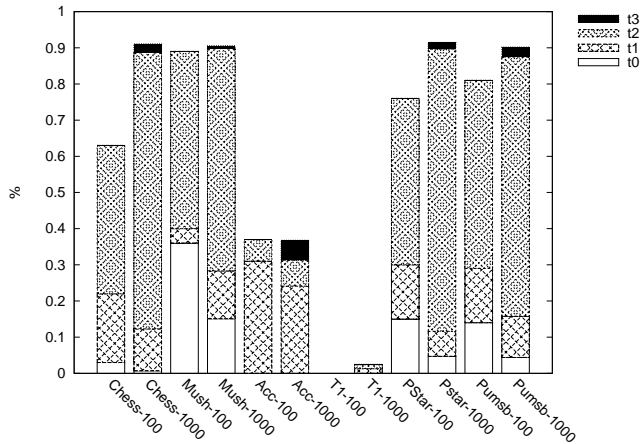}
\caption{Proportion of specious rules and their subtypes among the top-100 and top-1000 rules. Sets Plants and T4 are excluded because they contained no specious rules.} 
\label{specdistr}
\end{figure}

Figure \ref{specdistr} shows that the proportion of specious rules
varies greatly among the data sets. Chess, Mushroom, Accidents,
PumsbStar and Pumsb contained a large proportion of specious rules,
while Plants and T40I10D100K did not contain any and T10I4D100K only
sparsely (none among the top-100 rules). Interestingly, Plants,
T10I4D100K and T40I10D100K were the only sets of occurrence data,
listing either U.S.A. states where plant species occurred or product
items that occurred in market baskets. When the top-1000 rules were
studied, over 90\% of the 1000 best rules were specious in Chess,
Mushroom, PumsbStar, and Pumsb. In these data sets, the most common
subtype was 2 and even 78\% of rules could express Yule-Simpson's paradox. 
Accidents differed from these by having predominantly
type 1 specious rules. Type 0 rules occurred most commonly in
Mushroom. Type 3 rules were relatively rare and they did not occur
among the top-100 rules at all. One explanation is the cautious threshold,
and a larger threshold caught some type 3 rules among the top-100 rules in
Pumsb and PumsbStar. Another explanation is that the same rule can be
specious by many alternative rules and thus fall into different
types. This was demonstrated by checking possible $X\rightarrow C{=}c$ and $X\rightarrow C{\neq}c$ rules in a reversed order, which resulted in some specious rules  changing
their type to 0.

Table \ref{specrules100} shows that non-specious rules were in general better
than specious rules,
excluding type 0 rules. They necessarily had larger $MI$, because a rule
can become specious only by a better rule with the given goodness
measure. However, non-specious rules had also larger frequency and
leverage in all sets except Accidents. Among the top-100 rules, they had
also larger precision $p_{C|Q}$, but in the top-1000, the situation was
slightly different (in PumbStar and the T1 set, specious rules had higher
precision). In addition, specious rules were on average more complex
than non-specious rules. However, the difference was small in
Accidents, where many specious rules were generalizations of
non-specious rules. These results demonstrated that speciousness
pruning produced simpler and better 
rules, which is itself
desirable. We note that type 0 specious rules can well be better than
non-specious rules (e.g., in the top-1000 Mushroom $\avg(MI)$=4297)
because they are often equivalent with the very best rules.

An analysis of conditional leverages revealed 
that on average, $\delta_1$ was zero, but $\delta_2$ was often negative. This
means that in specious rules, $Q$ and $C{=}c$ tended to be negatively
associated given $\neg X$. The average $MI_S$ values were always
non-positive. Once again, Accidents differed from the others with
extremely low $MI_S$. A likely reason is that Accidents was also the
largest data set, with most frequent rules, which also tend to have
large $MI$. On the other hand, Accidents had smallest
conditional leverages (strong negative conditional dependencies) which
led to sign reversal and low $MI_S$-values. The mediating rules
$X\rightarrow Q{=}q$ were always strong, as expected, with
$\avg(\delta(X,Q{=}q))>\avg(\delta(Q,C{=}c))$. Still, $X$ and $Q$ had
relatively few common attributes, except in Accidents, where type 1
was common. For the top-100 rules, the average $p_B$ values were
large ($p_B\approx 1.00$), except for
Accidents $p_B\approx 0.91$. For the top-1000 rules, the $p_B$ values
were slightly smaller, $p_B$ ranging from 0.99 (Mushroom) to 0.88
(Accidents). With the top-1000, the minimum $p_B$ values were
substantially smaller, ranging from 0.15 (Accidents) to 0.35 (T1),
even if the maximum $MI_S$ threshold was the same. Still, none of them
was significant even in a traditional sense (where $\alpha=0.05$ or $\alpha=0.01$) and it is unlikely that
any true dependencies had been pruned out as specious.

{\footnotesize
\begin{table*}[!ht]
\caption{Mean values of statistics over all non-equivalent specious rules among the top-100 and top-1000 non-redundant rules with $MI$. Mean statistics for non-specious rules in parentheses.}
\label{specrules100}
\begin{center}
\begin{tabular}{|l|rr|rr|rr|rr|rr|rr|rr|rr|rr|}
\hline
Data&\multicolumn{2}{|c|}{$MI$}&\multicolumn{2}{|c|}{$fr$}&\multicolumn{2}{|c|}{$p_{C|Q}$}&\multicolumn{2}{|c|}{$p_{\neg C|\neg Q}$}&\multicolumn{2}{|c|}{$\delta$}&\multicolumn{2}{|c|}{$|Q|$}\\
\hline
\multicolumn{13}{|l|}{Top-100}\\
Plants&--&(8551)&--& (3535)&--& (0.86)&--& (0.95)&--& (0.120)&--&(1.3)\\
Chess&1338&(2159)&1287&(1819)&0.79&(0.93)&0.92&(0.98)&0.163&(0.176)&3.5&(1.2)\\
Mush&5622&(6483)&1771&(2350)&0.96&(0.99)&1.00&(0.98)&0.167&(0.206)&3.3&(0.6)\\
Acc&119793&(144940)&142478&(106631)&0.92&(0.95)&0.83&(0.89)&0.420&(0.317)&2.9&(2.3)\\
T1&--&(4339)&--&(730)&--&(0.89)&--&(0.99)&--&(0.007)&--&(1.4)\\
T4&--&(8064)&--&(1560)&--&(0.95)&--&(0.98)&--&(0.015)&--&(2.8)\\
PStar&46515&(46987)&23611&(24617)&0.99&(1.00)&0.99&(0.99)&0.245&(0.245)&2.3&(1.0)\\
Pumsb&47266&(47685)&23360&(24874)&0.99&(1.00)&1.00&(1.00)&0.247&(0.246)&2.6&(1.0)\\
\hline
\multicolumn{13}{|l|}{Top-1000}\\
Plants&--&(6186)&--&(2799)&--&(0.81)&--&(0.93)&--&(0.095)&--&(1.6)\\
Chess&946&(1438)&1138&(1601)&0.67&(0.87)&0.89&(0.88)&0.129&(0.141)&3.8&(2.3)\\
Mush&3707&(4091)&1802&(2622)&0.89&(0.95)&0.92&(0.86)&0.144&(0.161)&3.0&(2.2)\\
Acc&36784&(43751)&110655&(88285)&0.76&(0.79)&0.63&(0.68)&0.326&(0.259)&2.6&(2.4)\\
T1&2604&(2950)&448&(473)&0.95&(0.89)&0.99&(0.99)&0.004&(0.005)&2.5&(1.6)\\
T4&--&(7449)&--&(1284)&--&(0.95)&--&(0.99)&--&(0.013)&--&(2.9)\\
PStar&39428&(42149)&16000&(21975)&0.98&(0.98)&0.99&(0.98)&0.203&(0.232)&3.9&(2.0)\\
Pumsb&41208&(41988)&19276&(21655)&0.98&(0.98)&0.99&(0.98)&0.225&(0.232)&3.8&(2.6)\\
\hline
\end{tabular}
\end{center}
\end{table*}
}

\section{Conclusions}

The problem of Yule-Simpson's paradox and other specious dependencies
have bothered statisticians and empirical scientists for more than a
century. Still, no efficient method for detecting them has been known. In this research, we approached the problem from a new
perspective, in the context of statistical association rule
discovery. We showed that specious rules offer a unifying framework
for many types of undesirable, redundant or misleading association
rules. We introduced new theoretical properties that enable effective
identification of possible confounding factors without testing all of
the exponentially many possibilities. These results offer a remarkable
improvement to currently known weak conditions. Then we showed how
the properties can be implemented in the pattern discovery context, as
an efficient generic algorithm that discovers the most significant,
non-specious rules with any order-homomorphic significance
measures. Preliminary experiments with mutual information demonstrate
that specious rules and Yule-Simpson's paradox are indeed a common and
serious problem in association rule discovery. However, with
speciousness detection, association rules can reveal the real
dependency structure of data and help to avoid misleading conclusions.

\section*{Acknowledgments}
This research has been supported by the Academy of Finland under grant 
307026 and by the Australian Research Council under grant DP140100087.

\bibliographystyle{plain}

\begin{thebibliography}{10}
\bibitem{Alin2010}
A.~Alin.
\newblock Simpson’s paradox.
\newblock {\em Wiley Interdisciplinary Reviews: Computational Statistics},
  2:247--250, 2010.

\bibitem{birch}
M.~W. Birch.
\newblock The detection of partial association, i: The 2 × 2 case.
\newblock {\em J.\ Royal Statistical Society. Series B
  (Methodological)}, 26(2):313--324, 1964.

\bibitem{FIMI}
Frequent itemset mining dataset repository.
\newblock Administered by B. Goethals. \url{http://fimi.ua.ac.be/data/}.

\bibitem{kingfkais}
W.~H{\"a}m{\"a}l{\"a}inen.
\newblock Kingfisher: an efficient algorithm for searching for both positive
  and negative dependency rules with statistical significance measures.
\newblock {\em Knowledge and Information Systems}, 32(2):383--414, 2012.

\bibitem{HuangWebb2005}
S.~Huang and G.I. Webb.
\newblock Pruning derivative partial rules during impact rule discovery.
\newblock In {\em Proc.\  9th Pacific-Asia Conf.\ Advances in Knowledge Discovery and Data
  Mining}, pp.\ 71--80. Springer, 2005.

\bibitem{causalrules}
Z.~Jin, J.~Li, L.~Liu, T.~D. Le, B.{-}Y. Sun, and R.~Wang.
\newblock Discovery of causal rules using partial association.
\newblock In {\em Proc.\ 12th {IEEE} Int.\ Conf.\  Data Mining}, pp.\ 309--318. IEEE, 2012.

\bibitem{lewis1962}
B.~N. Lewis.
\newblock On the analysis of interaction in multi-dimensional contingency
  tables.
\newblock {\em J.\ Royal Statistical Society. Series A (General)},
  125(1):88--117, 1962.

\bibitem{MLrep}
M.~Lichman.
\newblock {UCI} machine learning repository, 2013.

\bibitem{lindley}
D.V. Lindley and M.R. Novick.
\newblock The role of exchangeability in inference.
\newblock {\em Annals of Statistics}, 9:45--58, 1981.

\bibitem{pasquier99}
N.~Pasquier, Y.~Bastide, R.~Taouil, and L.~Lakhal.
\newblock Discovering frequent closed itemsets for association rules.
\newblock In {\em Proc.\ 7th Int.\ Conf.\ Database
  Theory}, volume 1540 of {\em LNCS},
  pp.\ 398--416. Springer-Verlag, 1999.

\bibitem{simpson1951}
E.~H. Simpson.
\newblock The interpretation of interaction in contingency tables.
\newblock {\em J.\ Royal Statistical Society. Series B
  (Methodological)}, 13(2):238--241, 1951.

\bibitem{webbml}
G.I. Webb.
\newblock Discovering significant patterns.
\newblock {\em Machine Learning}, 68(1):1--33, 2007.

\bibitem{yule1903}
G.U. Yule.
\newblock Notes on the theory of association of attributes in statistics.
\newblock {\em Biometrika}, 2:121--134, 1903.
\end{thebibliography}

\end{document}